\documentclass{article}
\usepackage{spconf,amsmath,graphicx, amsfonts,multirow, caption, subcaption, mathtools,amsfonts, amsthm}
\usepackage{array}
\usepackage{textcomp}
\usepackage{stfloats}
\usepackage{url}
\usepackage[colorlinks=true, linkcolor=blue, citecolor=blue, urlcolor=blue]{hyperref}

\usepackage{verbatim}
\usepackage{algorithm}
\usepackage[noend]{algpseudocode}
\usepackage{adjustbox}

\DeclarePairedDelimiter\abs{\lvert}{\rvert}
\DeclarePairedDelimiter\norm{\lVert}{\rVert}

\DeclareMathOperator*{\argmin}{arg\,min}
\DeclareMathOperator*{\argmax}{arg\,max}
\DeclarePairedDelimiter\ceil{\lceil}{\rceil}

\newtheorem{theorem}{Theorem}[section]

\title{Towards joint graph learning and sampling set selection from data}

\name{Shashank N. Sridhara,  Eduardo Pavez, Antonio Ortega,   \thanks{This work was funded in part by the National Science Foundation (NSF CNS-1956190).}}
\address{University of Southern California, Los Angeles, CA, USA}

\long\def\comment#1{}

\newfont{\bbb}{msbm10 scaled 700}

\newfont{\bb}{msbm10 scaled 1000}

\newcommand{\ev}{{\bf e}}
\newcommand{\fv}{{\bf f}}

\newcommand{\nv}{{\bf n}}

\newcommand{\uv}{{\bf u}}

\newcommand{\xv}{{\bf x}}
\newcommand{\yv}{{\bf y}}
\newcommand{\zv}{{\bf z}}
\newcommand{\zerov}{{\bf 0}}

\newcommand{\Am}{{\bf A}}
\newcommand{\Bm}{{\bf B}}

\newcommand{\Dm}{{\bf D}}

\newcommand{\Gm}{{\bf G}}
\newcommand{\Hm}{{\bf H}}
\newcommand{\Id}{{\bf I}}

\newcommand{\Lm}{{\bf L}}
\newcommand{\Mm}{{\bf M}}

\newcommand{\Qm}{{\bf Q}}

\newcommand{\Sm}{{\bf S}}

\newcommand{\Um}{{\bf U}}

\newcommand{\Xm}{{\bf X}}

\newcommand{\Zm}{{\bf Z}}

\newcommand{\Ec}{{\cal E}}

\newcommand{\Gc}{{\cal G}}

\newcommand{\Nc}{{\cal N}}
\newcommand{\Oc}{{\cal O}}

\newcommand{\Rc}{{\cal R}}
\newcommand{\Sc}{{\cal S}}

\newcommand{\Vc}{{\cal V}}

\newcommand{\deltav}{\hbox{\boldmath$\delta$}}

\newcommand{\Sigmam}{\hbox{\boldmath$\Sigma$}}

\newcommand{\Omegam}{\hbox{\boldmath$\Omega$}}

\begin{document}
\ninept
\maketitle
\begin{abstract}
We explore the problem of sampling graph signals in scenarios where the graph structure is not predefined and must be inferred from data. In this scenario, existing approaches rely on a two-step process, where a graph is learned first,  followed by sampling. More generally, graph learning and graph signal sampling have been studied as two independent problems in the literature. 
This work provides a foundational step towards jointly optimizing the graph structure and sampling set. Our main contribution, \textit{Vertex Importance Sampling} (VIS), is to show that the sampling set can be effectively determined from the vertex importance (node weights) obtained from graph learning. 
We further propose \textit{Vertex Importance Sampling with Repulsion} (VISR), a greedy algorithm where spatially-separated ``important'' nodes are selected to ensure better reconstruction. Empirical results on simulated data show that sampling using VIS and VISR leads to competitive reconstruction performance and lower complexity than the conventional two-step approach of graph learning followed by graph sampling.

\end{abstract}
\begin{keywords}
Graph signal sampling, Graph learning, Graph signal reconstruction, D-optimality.
\end{keywords}
\section{Introduction}
\label{sec:intro}
Graph signal processing (GSP) provides a powerful framework for analysis, denoising, sampling, and interpolation of signals defined on graphs \cite{ortega2022_GSP, chen2015_graphsampling, narang2013_graphinterpolation}. \textit{Sampling graph signals}
is a fundamental problem in GSP \cite{chen2015_graphsampling, anis2016_spectralproxies, jayawant2021_avm, bai2020_bsgda}, where the goal is selecting a subset of graph nodes to reconstruct a smooth signal from the signal samples.

In some applications, the graph structure is not predefined, although the data (signal values on vertices) is available \cite{dong2019_learningL, Eglimez2017_graphlearn}. In such cases, the graph structure must first be inferred from the data before graph signal sampling. Learning graph structure from data is a well-studied topic in the GSP literature \cite{dong2019_learningL, Eglimez2017_graphlearn, pavez2019_Llearn, girault2023_LQlearn}. However, when the objective is to determine a sampling set without a predefined graph, the conventional two-step process—first learning the graph and then performing sampling—can become computationally expensive. Our work specifically focuses on obtaining a sampling set in scenarios where graph structure is not readily available.

Graph learning from data can be formulated as a Maximum A-Posteriori (MAP) parameter estimation of Gaussian-Markov Random fields (GMRF), whose inverse covariance is graph Laplacian matrix \cite{dong2019_learningL, Eglimez2017_graphlearn, pavez2019_Llearn}. Specifically, learning a Diagonally Dominant Graph Laplacian (DDGL) allows for the simultaneous optimization of both edge weights and \textit{vertex importance}, 
the latter represented by a diagonal matrix added to the Combinatorial Graph Laplacian (CGL) \cite{Eglimez2017_graphlearn, girault2023_LQlearn}. 
Further, the learned vertex importance values can be interpreted as an inner-product matrix to define a generalized \textit{Graph Fourier Transform} (GFT) \cite{girault2018_IAGFT, girault2023_LQlearn}.

Given a graph and its corresponding graph variation operator, e.g., graph Laplacian, graph signal sampling can be formulated as Minimum Mean Squared error Estimation (MMSE) of GMRF signals from the samples \cite{Gadde2015_ProbSampling}. Several graph sampling algorithms have been developed within this framework \cite{bai2020_bsgda, Gadde2015_ProbSampling, sridhara2024_pcsampling}.

In this work, we show that graph learning from data can provide insights for determining the sampling set. In particular, the common assumption in both the formulations, i.e.,  graph signals are realizations of independent identically distributed (i.i.d.) zero-mean multivariate Gaussian distributions with inverse covariance equal to the graph Laplacian,  can be leveraged for simultaneous optimization of graph structure (edge set and edge weights) and the sampling set. 
To achieve this, we formulate sampling set selection using the D-optimal objective function. 
We identify similarities between the D-optimal sampling and graph learning objectives. We then show that the sampling set can be directly obtained from the \textit{vertex importances} learned from finding the optimal diagonally dominant graph Laplacian \cite{girault2023_LQlearn}. However, the sampling set based solely on vertex importance could result in spatially close samples, leading to subpar reconstruction accuracy at higher sampling rates. To address this limitation, we propose a low-complexity sampling algorithm that selects ``important'' nodes that are spatially separated from each other. 

Our contributions are: i)  Two algorithms, \textit{Vertex Importance sampling (VIS)} and \textit{Vertex Importance Sampling with Repulsion (VISR)}, providing approximate solutions to the D-optimal sampling objective  using learned vertex importance, 
ii)  a greedy algorithm for the D-optimal sampling similar to that of  \cite[Algorithm 1]{krause2008_sensor} to compare the sampling sets from VIS and VISR, and iii) an empirical evaluation showing that using the DDGL for reconstruction 
leads to better signal reconstruction than using only the CGL. 

The rest of the paper is organized as follows. GSP preliminaries and notations are given in \autoref{sec:preliminaries}. Sampling formulation and a greedy solution to the D-optimal objective are given in \autoref{sec:sampling_formulation}. We propose VIS and VISR algorithms in \autoref{sec:sampling_from_q}. Experimental results and conclusions are in  \autoref{sec:experiments} and \autoref{sec:conclusion}, respectively.

\section{Notations and Background}
\label{sec:preliminaries}
\subsection{Graph signal processing}
Throughout the paper, we represent sets using calligraphic uppercase, $\Sc$, 
matrices by bold uppercase, $\Xm$, vectors by bold lowercase, $\xv$, and scalars by plain lowercase, $x$.
A graph $\Gc = (\Vc, \Ec)$ consists of a set of vertices $\Vc$ of size $N$ and an edge set $\Ec \subseteq \Vc \times \Vc$. A graph signal is a vector $\fv \in \mathbb{R}^{N}$, whose $i$-th entry corresponds to the signal at the $i$-th node of the graph. We define the edge weight between vertices $i$ and $j$ as $w_{ij} > 0$. In this work, we consider only undirected graphs, so we have $w_{ij} = w_{ji}$, and we have a symmetric adjacency matrix, $\Am \in \mathbb{R}^{N \times N}$, whose $(i, j)$-th entry is $w_{ij}$, if $(i, j) \in \Ec$ and $0$ otherwise. The degree of a vertex is denoted by $d_{i} = \sum_{(i, j) \in \Ec} w_{ij}$, and the diagonal degree matrix is defined as $\Dm = \mathrm{diag}(d_{1}, d_{2} \cdots d_{N})$. The combinatorial graph Laplacian, $\Lm := \Dm - \Am$ is a symmetric, positive semi-definite matrix \cite{ortega2022_GSP}.

The irregularity-aware graph Fourier transform (GFT) introduced in  \cite{girault2018_IAGFT} is based on a positive semi-definite graph variation operator $\Mm \succeq 0$, which measures graph signal smoothness, and a Hilbert space with inner product defined as $\langle \xv, \yv\rangle_{\Qm} = \yv^{T} \Qm \xv$, with induced Q-norm $\norm{\xv}_{\Qm}^{2} = \langle\xv, \xv \rangle_{\Qm}$, where $\Qm \succ 0$. The generalized GFT for the $\Qm$-inner product is the $(\Mm, \Qm)-$GFT. 
The graph-Fourier basis $\Um =  [\uv_{1}, \uv_{2}, \cdots \uv_{N}]$ is now orthogonal with respect to the $\Qm$-inner product, i.e., $\Um^{H} \Qm \Um = \mathbf{I}$. The eigenvalues corresponding to $(\Mm, \Qm)-$GFT, $0  = \lambda_{1} \leq \lambda_{2}, \cdots, \leq \lambda_{N}$ represent the graph frequencies \cite{girault2018_IAGFT}.

\subsection{Graph and vertex importance learning}
Assuming graph signals $\{\fv^{(k)}\}_{k}$ are realizations of a zero-mean, i.i.d.~multivariate Gaussian distribution, the goal of graph learning is to model the signals with a graph whose combinatorial graph Laplacian is such that the data satisfies graph Wide Sense Stationarity (gWSS) with gPSD $\gamma(0) = 0$ and $\gamma(\lambda) = 1/\lambda$, when $\lambda > 0$ \cite{pavez2019_Llearn, girault2023_LQlearn}. Using the $(\Lm, \mathbf{I})-$GFT, this translates into a covariance matrix $\Sigmam = \Lm^{\dagger}$, where $\Lm^{\dagger}$ denotes the pseudo-inverse of $\Lm$. Under this setting, the graph learning problem is traditionally formulated as a maximum likelihood estimation, minimizing the cost function \cite{pavez2019_Llearn}: 
\begin{equation}
    \label{eqn:L_learning}
    J(\Lm) = -\mathrm{log det}(\Lm + \frac{1}{N}\mathbf{1} \mathbf{1}^{T}) + \mathrm{tr}(\Lm \Sm),
\end{equation}
where $\Sm = \frac{1}{k} \sum_{k}\fv^{(k)}\fv^{(k)\top}$ is the empirical covariance matrix. The cost function in \eqref{eqn:L_learning} can be efficiently optimized using the coordinate minimization approach proposed in \cite{pavez2019_Llearn}.

The joint graph and vertex importance learning proposed in \cite{girault2023_LQlearn} requires solving the following optimization:
\begin{equation}
    \label{eqn:LQ_learning}
    J(\Lm, \Qm) = -\mathrm{log det}(\Lm + \Qm) + \mathrm{tr}((\Lm + \Qm)\Sm),
\end{equation}
where $\Qm$ is a diagonal matrix with $\Qm_{ii}>0$ representing vertex importance. 
The cost function \eqref{eqn:LQ_learning} can be solved using coordinate minimization by updating edges weights and \textit{vertex importances} iteratively \cite{girault2023_LQlearn}. 
Learning a diagonally dominant graph Laplacian (DDGL) model can be viewed as learning a combinatorial graph Laplacian $\Lm$ and an inner-product matrix $\Qm$. 
Through joint optimization, we learn the graph such that the data is gWSS with gPSD $\gamma(\lambda) = (1 + \lambda)^{-1}$ when using $(\Lm, \Qm)-$GFT. The covariance matrix with respect to $(\Lm, \Qm)-$GFT is given by $\Sigmam = [\Lm + \Qm]^{-1}$ \cite[Theorem 1]{girault2023_LQlearn}.

\section{Graph Signal Sampling Formulation}
\label{sec:sampling_formulation}
Our graph signal sampling formulation is based on the same signal model used for graph learning, i.e., $\fv \sim \Nc(\zerov, \Sigmam = \Omegam^{\dagger})$, where $\Omegam$ is a positive semi-definite graph operator (e.g. $\Lm$ or $\Lm + \Qm$). We define the sampling process as \cite{chen2015_graphsampling, jayawant2021_avm, anis2016_spectralproxies}
\begin{equation*}
    \yv_{\Sc} = \fv_{\Sc} + \nv = \Id_{\Sc}^{\top}\fv + \nv,
\end{equation*}
where $\nv \in \mathbb{R}^{\abs{\Sc}}$ is the noise introduced during sampling and $\Id_{\Sc}$ 
is $N\times \abs{\Sc}$ sampling matrix and $\Hm = \Id_{\Sc}\Id_{\Sc}^\top$ is a diagonal matrix: 
\begin{equation}
\label{eqn:sampling_upsampling}
    \Hm_{ii} =  \begin{dcases}
    1, & \text{if }  i \in \Sc\\
    0,              & \text{otherwise}.
\end{dcases}
\end{equation}
We assume the signal is reconstructed using the graph Laplacian regularization  (GLR) approach,  given by \cite{bai2020_bsgda, sridhara2024_pcsampling}: 
\begin{equation}
    \label{eqn:reconstruction}
    \hat{\fv} = \argmin_{\fv} \norm{\Id_{\Sc}^{\top}\fv - \yv_{\Sc}}_{2}^{2} + \mu \fv^{\top}  \Omegam  \fv, 
\end{equation}
whose closed-form solution is  \cite{bai2020_bsgda}:
\begin{equation}
    \label{eqn:closedform_soln}
    \hat{\fv} = (\Hm + \mu \Omegam)^{-1} \Id_{\Sc}\yv_{\Sc}.
\end{equation}
Next, we define an objective function for sampling. 
\subsection{Sampling objective}
From \eqref{eqn:closedform_soln}, $\hat{\fv}$ is the solution system of linear equations where $\Bm := (\Hm + \mu \Omegam)$ is the coefficient matrix. Since $\Bm$ is a positive-definite matrix for any given positive semi-definite graph operator $\Omegam$ \cite{bai2020_bsgda}, the inverse of $\Bm$ always exists. Assuming that $\nv \sim (\zerov, \Id)$, we can write the error vector and error covariance matrix corresponding to the reconstruction in \eqref{eqn:closedform_soln} as: 
\begin{align}
    \ev = (\hat{\fv} - \fv) = - \mu\Bm^{-1} \Omegam \fv + \Bm^{-1} \Id_{\Sc} \nv \\
    \mathbb{E}[\ev \ev^{\top}] = \Bm^{-1} - (\mu - \mu^{2}) \Bm^{-1}\Omegam \Bm^{-1}.
\end{align}
\begin{theorem}
\label{thrm:approx_error_covariance}
    For small $\mu$, the error covariance can be approximated as 
    \begin{equation}
        \mathbb{E}[\ev \ev^{\top}] = \Bm^{-1} - (\mu - \mu^{2}) \Bm^{-1}\Omegam \Bm^{-1} \approx (\Hm + \gamma \Omegam)^{-1}, 
    \end{equation}
    where $\gamma = 2\mu - \mu^2$
\end{theorem}
\begin{proof}
The proof relies on the fact that $\norm{\Bm^{-1}(\mu - \mu^{2}) \Omegam} < 1$. The first-order approximation of the Neumann series on matrices (Th. 5.6.12 of \cite{horn2012_matrix}) gives the desired result. Due to space limitations, we omit the detailed proof.
\end{proof}
Different sampling objectives can be defined to measure the reconstruction accuracy as a function of error covariance \cite{anis2016_spectralproxies}. In our work, we consider the widely used D-optimality criterion, which minimizes the determinant of the error covariance matrix \cite{jayawant2021_avm}. The D-optimal sampling objective based on the approximated error covariance can be formulated as
\begin{equation}
    \label{eqn:d_opt_obj}
    \Hm^{\star} = \argmin_{\Hm_{ii} = \{0, 1\}, \sum_{k = 0}^{N} \Hm_{ii} = \abs{\Sc}} -\mathrm{logdet} [\Hm + \gamma \Omegam], 
\end{equation}
an NP-hard combinatorial optimization problem  \cite{jayawant2021_avm}. 
In what follows, we first present a greedy algorithm that approximately solves \eqref{eqn:d_opt_obj} and then a sampling algorithm based on the graph learning framework (\autoref{sec:sampling_from_q}).

\subsection{Greedy sampling set selection}
\label{subsec:greedy_sampl}
In our greedy approach, the first selected sample is,
\begin{equation*}
    i_{1}^{*} = \argmax_{j \in \mathcal{V}} \mathrm{det} [\deltav_{j} \deltav_{j}^{T} + \gamma \Omegam] = \argmax_{j \in \mathcal{V}}  \deltav_{j}^{T} \Omegam^{-1} \deltav_{j},
\end{equation*}
where $\deltav_{j} \in \mathbb{R}^{N}$ is the delta function at node $j$, with value $1$ at the $j^{\mathrm{th}}$ node and $0$ elsewhere. 
The second equality is due to the Matrix-Determinant lemma \cite{horn2012_matrix}. Thus, the first iteration chooses the node corresponding to the index of the maximum diagonal element of $\Omegam^{-1}$, with   $\Oc(N^{3})$ computations required to obtain $\Omegam^{-1}$. 

At the $k^{\mathrm{th}}$ iteration, $\abs{S} = k-1$, the next sample is chosen as: 
\begin{equation}
   i_{k}^{*} = \argmax_{j \in \Sc^{c}} \deltav_{j}^{T} \Gm_{k-1}^{-1} \deltav_{j},
\end{equation}
where $\Gm_{k-1} = (\Id_{\Sc}\Id_{\Sc}^{T} + \gamma \Omegam)$. Once the $k^{\mathrm{th}}$ sample is updated, the inverse of $\Gm_{k}$ can be computed using the Sherman-Morrison formula:
\begin{equation}
    \label{eqn:SM_inv}
    \Gm_{k}^{-1} = \Gm_{k-1}^{-1}  -  \frac{\Gm_{k-1}^{-1} \deltav_{i_{k}^{*}} \deltav_{i_{k}^{*}}^{T} \Gm_{k-1}^{-1}}{ 1 + \deltav_{i_{k}^{*}}^{T} \Gm_{k-1}^{-1} \deltav_{i_{k}^{*}}^{T}}.
\end{equation}
Therefore, we do not need to compute the matrix inverse and determinant at subsequent iterations. Instead, at each iteration we select a sample $j$ from $\Sc^{c}$ such that $\deltav_{j}^{T} \Gm_{k}^{-1} \deltav_{j}$ is maximum, while the inverse of $\Gm_{k}$ can be updated using \eqref{eqn:SM_inv}.

\section{Sampling based on  vertex importances}
\label{sec:sampling_from_q}

\begin{algorithm}[t]
\caption{Vertex Importance Sampling with repulsion (VISR)}\label{algo:sampling_with_repulsion}
\begin{algorithmic}[1]
\Procedure{Sampling}{$\Am, \Dm, \Qm, k$}
\State $\Zm(p) \gets \sum_{l=1}^{p}(\Dm^{-1} \Am)$
\Comment{p-hop localized filter}
\State $\zv_{i}^{(p)} \gets \Zm(p)\deltav_{i},  \forall i$ \Comment{Columns of $\Zm(p)$}
\State $[q_{1} \ldots q_{N}] \gets diag(\Qm)$ \Comment{Vertex importances}

\While{$\abs{\Sc} \leq  k$} 
\State $x^{*}\gets \argmax_{x \in \Sc^{c}}  (q_{x} - \sum_{y \in \Sc} \langle \zv_{x}^{(p)}, \zv_{y}^{(p)} \rangle)$
\State $\Sc \gets \Sc \cup x^{*}$

\EndWhile
\State \textbf{return} $\Sc$
\EndProcedure
\end{algorithmic}
\end{algorithm}
When $\Omegam = (\Lm + \Qm)$, 
we can rewrite \eqref{eqn:d_opt_obj} as 
\begin{equation*}
    \Hm^{\star} = \argmin_{\Hm_{ii} = \{0, 1\}, \sum_{k = 0}^{N} \Hm_{ii} = \abs{\Sc}} -\mathrm{logdet} [\Hm + \gamma (\Qm + \Lm)],
\end{equation*}
which can be rewritten by factoring $\gamma$ as
\begin{equation}
    \label{eqn:sampl_obj_LQ}
    \Hm^{\star} = \argmin_{\Hm_{ii} = \{0, 1\}, \sum_{k = 0}^{N} \Hm_{ii} = \abs{\Sc}} -\mathrm{logdet} [(\frac{1}{\gamma}\Hm +\Qm) + \Lm].
\end{equation}
Comparing the sampling and graph learning objectives, \eqref{eqn:sampl_obj_LQ} and \eqref{eqn:LQ_learning}, we observe that their first terms (log-determinant) are similar. At optimality, the trace term in \eqref{eqn:LQ_learning} becomes constant \cite{girault2023_LQlearn}. 
Therefore, due to Weyl’s theorem \cite[Corollary 4.3.9]{horn2012_matrix}, the indices $i$ of $\Hm_{ii} = 1$ that corresponds to maximizing the determinant in \eqref{eqn:sampl_obj_LQ} will be most likely the same indices corresponding to the highest values of vertex importance ($q_{i} = \Qm_{ii}$).
Based on this observation, we can formulate a \textit{Vertex Importance Sampling (VIS)}, sampling set selection, where the sampling set $\Sc_{k}$ of cardinality $k$ is chosen  as follows: 
\begin{align}
    \label{eqn:sampl_from_Q}
    \Sc_{k} = \{j_{1}, j_{2}, \ldots j_{k}\} \text{ s.t } \nonumber \\
    q_{j_{1}} \geq q_{j_{2}} \geq \ldots, q_{j_{k}} \geq  q_{i}, \, \forall i \notin \Sc_{k}.
\end{align}
This sampling set selection method requires no additional computation once the graph has been learned since it is directly obtained from learned \textit{vertex importance} values. 
However, since we sample every node independently---meaning the selection of a new node does not depend on the previously sampled nodes, the algorithm may select samples that are spatially close. We address this limitation by providing a low-complexity sampling algorithm based on vertex importance that incorporates \textit{repulsion} between selected samples.

\subsection{Vertex Importance Sampling with Repulsion}
We propose a low-complexity sampling algorithm that introduces repulsion between selected samples using a $p-$hop localized graph filter (\autoref{algo:sampling_with_repulsion}).
Following the intuition from  \cite{sridhara2024_pcsampling, sakiyama2019_edfree, jayawant2018_distancesampl}, we select the $k$ most ``informative'' nodes that are placed as far as possible from each other. Following \cite{sridhara2024_pcsampling}, we achieve this by defining a $p$-hop polynomial graph low-pass filter:
\begin{equation}
    \label{eqn:polynomial_filter}
    \Zm:= \sum_{l = 1}^{p} (\Dm^{-1}\Am)^{l}.
\end{equation}
The $i^{\text{th}}$ column of $\Zm(p)$ can be seen as a low-pass filtered delta function $\deltav_{i}$, localized at node $i$. For any $p$ we have:
\begin{equation}
 \label{eqn:p_order_approx}
\zv_{i}^{(p)} := \Zm(p) \deltav_{i} = (\Dm^{-1}\Am) \deltav_{i} + \cdots + (\Dm^{-1}\Am)^{p}\deltav_{i}.
\end{equation}
The parameter $p$ defines the localization of the polynomial graph filter $\Zm(p)$, i.e., 
when $p = 1$,  $\zv_{i}^{(1)} = (\Dm^{-1}\Am) \deltav_{i}$ which results in non-zero values only for the 1-hop neighbors of node $i$.  Similarly, $\zv_{i}^{(2)} = (\Dm^{-1}\Am) \deltav_{i} +(\Dm^{-1}\Am)^{2}\deltav_{i}$ contains non-zero values within 2-hops of $i$. In general, $\zv_{i}^{(p)}$ has non-zero values at nodes within the $p$-hop neighborhood of node $i$. 

We use the $p$-hop localization of $\zv_{i}^{(p)}$ to achieve separation between the selected samples. 
Particularly, by calculating the inner-products $\langle \zv_{i}^{(p)}, \zv_{j}^{(p)} \rangle, \forall i \neq j$, we prioritize selecting the most important nodes (based on learned vertex importance values  $q_{i}$), while also ensuring that the selected samples are spaced apart by at least $p$-hop neighbors. 
We adopt a greedy strategy similar to that of \cite{sakiyama2019_edfree, sridhara2024_pcsampling}, whereby, at each iteration, we add to the sampling set the point $x^{*}$ that solves:
\begin{equation}
    \label{eqn:isr_update}
    x^{*} = \argmax_{x \in \Rc}  (q_{x} - \sum_{y \in \Sc} \langle \zv_{y}^{(p)}, \zv_{x}^{(p)} \rangle).  
\end{equation}

\begin{figure}[t]
    \centering
    \includegraphics[width=0.7\linewidth]{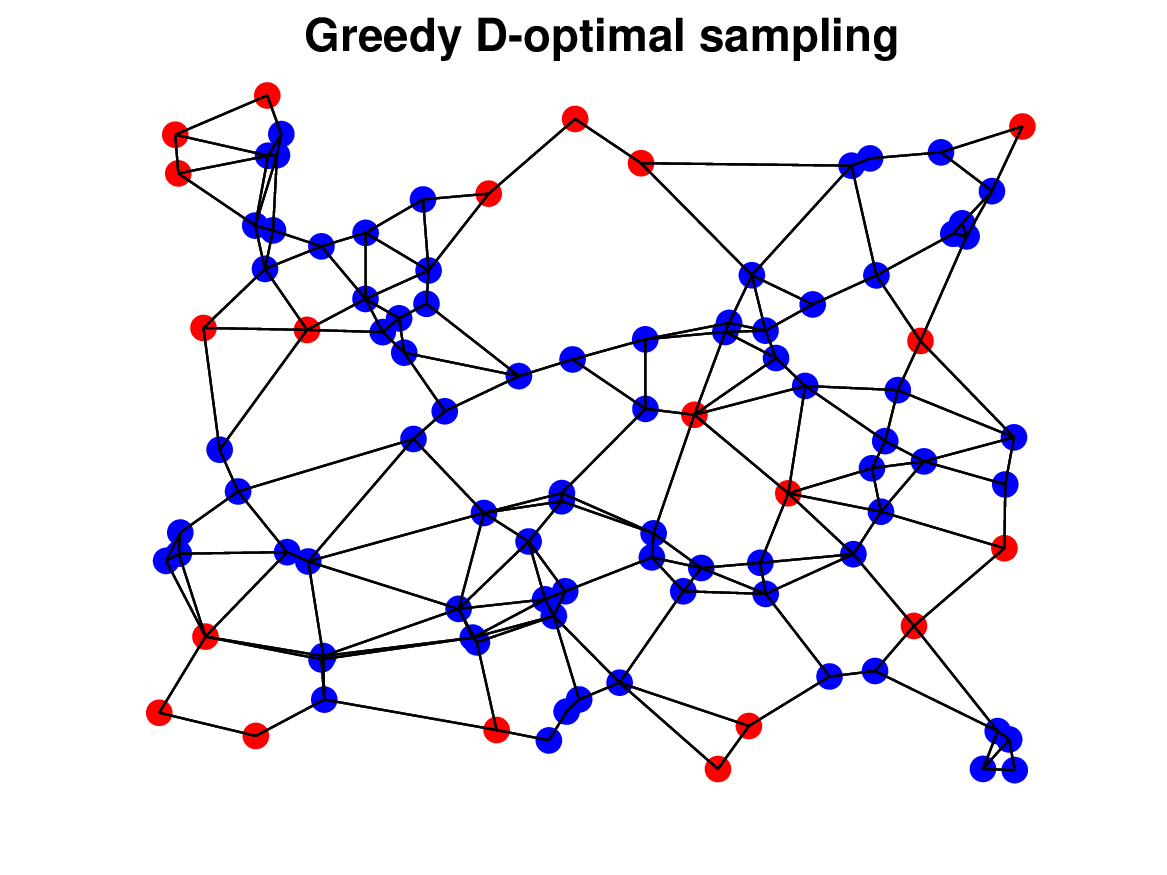}
    \caption{Vertices selected from the greedy solution to D-optimal sampling objective}
    \label{subfig:d_opt_sampl}
\end{figure}
To ensure sufficient repulsion, we set the value of $p$ to a small value when the sampling rate ($\alpha = k/N$) is high and to a large value when the sampling rate is low. We use the same  rule proposed in \cite[Section 3D]{sridhara2024_pcsampling} to select $p$ as
\begin{equation}
    \label{eqn:p_val}
    p = \ceil*{1/(2\alpha)}.
\end{equation}
This flexibility of selecting the parameter $p$ based on sampling rate makes  \autoref{algo:sampling_with_repulsion} scalable to large graphs. Contrary to existing sampling algorithms \cite{anis2016_spectralproxies, bai2020_bsgda, jayawant2021_avm}, the proposed VISR algorithm is computationally efficient because it uses learned vertex importance as an \textit{auxiliary information} and the repulsion is obtained by precomputed inner-products to obtain the sampling set. In \cite{anis2016_spectralproxies, jayawant2021_avm}, the inner products are computed in every iteration, making them computationally expensive. 

\section{Experiments}
\label{sec:experiments}

To experimentally validate our proposed sampling algorithms, VIS and VISR, we use synthetic data where node locations are obtained by uniformly sampling from the $[0,1] \times [0,1]$ Euclidean plane at $N = 100$ locations. These node locations are fixed across all experiments. 
We consider a Gaussian process with covariance matrix $\Sm$ to model the continuous process in space, where $\Sm_{ii} = 10, \forall i$ and off-diagonal entries $\Sm_{ij} = 10 \mathrm{exp}(-d_{ij}/r), \forall i \neq j$. Here $d_{ij}$ is the distance between nodes $i$ and $j$, and $r$ quantifies the range of signal correlation across the space. 
For all experiments, we fix $r = 0.02$. We generate $1000$ i.i.d.~realizations of signal $\fv$ and obtain its sample covariance matrix. We use this sample covariance to learn a combinatorial graph Laplacian $\Lm$ and a diagonally dominant graph Laplacian $(\Lm + \Qm)$ by solving \eqref{eqn:L_learning} and \eqref{eqn:LQ_learning}, respectively.
\begin{figure*}[t]
     \centering
        \begin{subfigure}[b]{0.33\textwidth}
            \centering
            \includegraphics[width=\linewidth]{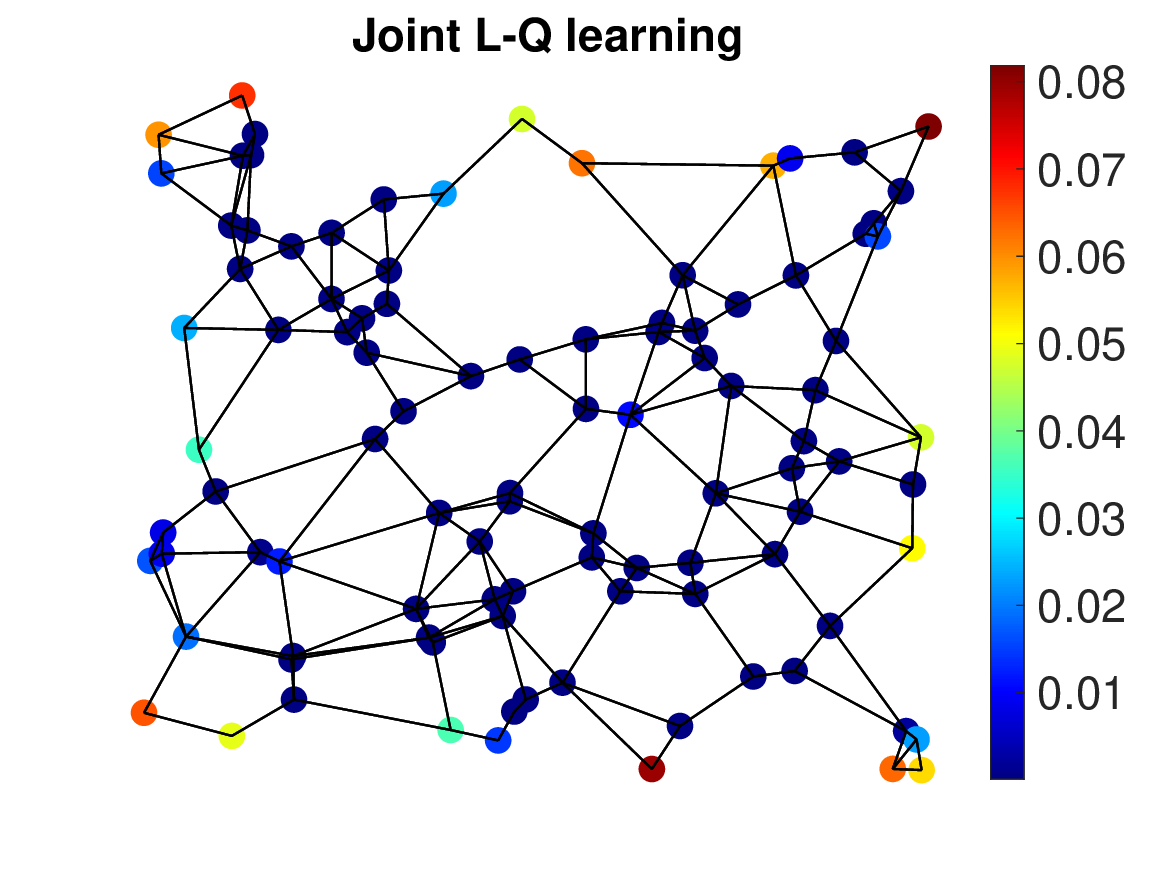}
            \caption{}
            \label{subfig:joint_lq_learning}
        \end{subfigure}
     \begin{subfigure}[b]{0.33\textwidth}
         \centering
         \includegraphics[width=\linewidth]{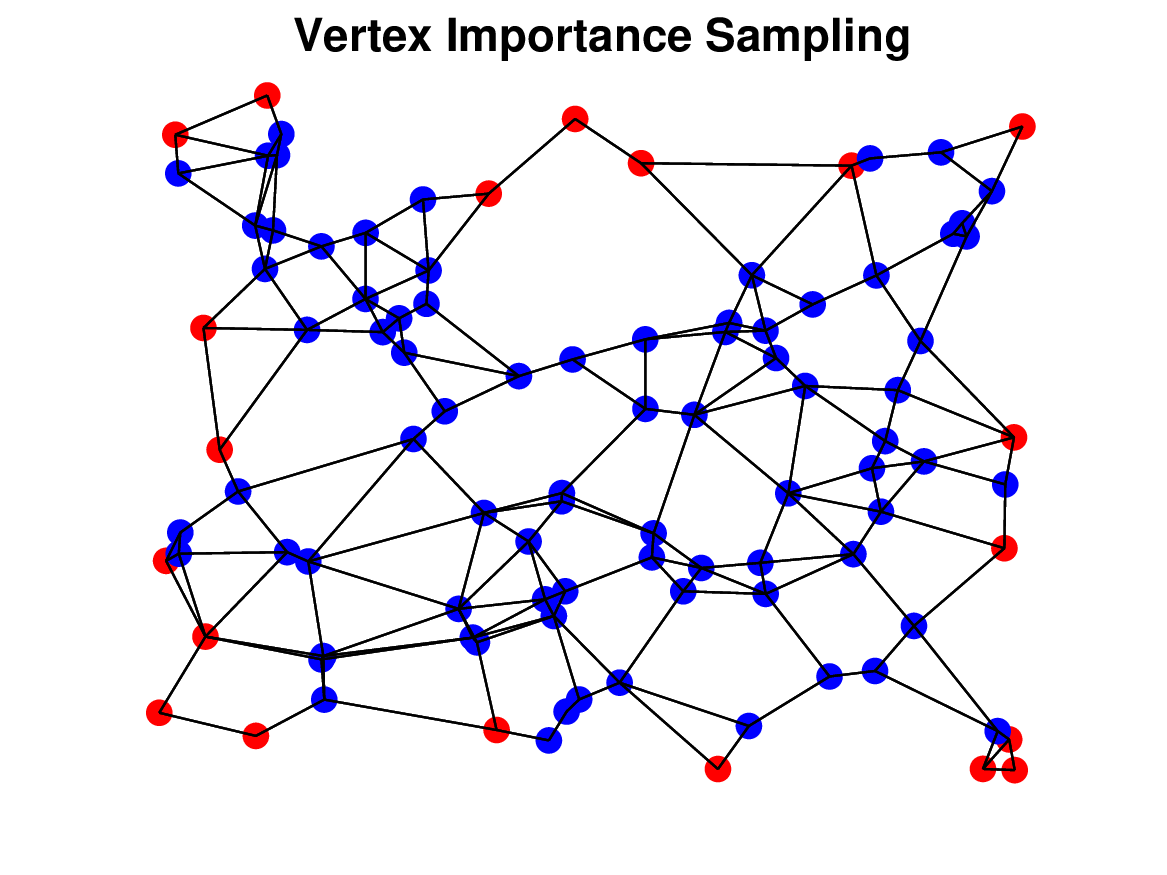}
         \caption{$\abs{\Sc} = 20$}
         \label{subfig:sampling_from_Q}
     \end{subfigure}
      \begin{subfigure}[b]{0.33\textwidth}
         \centering
         \includegraphics[width=\linewidth]{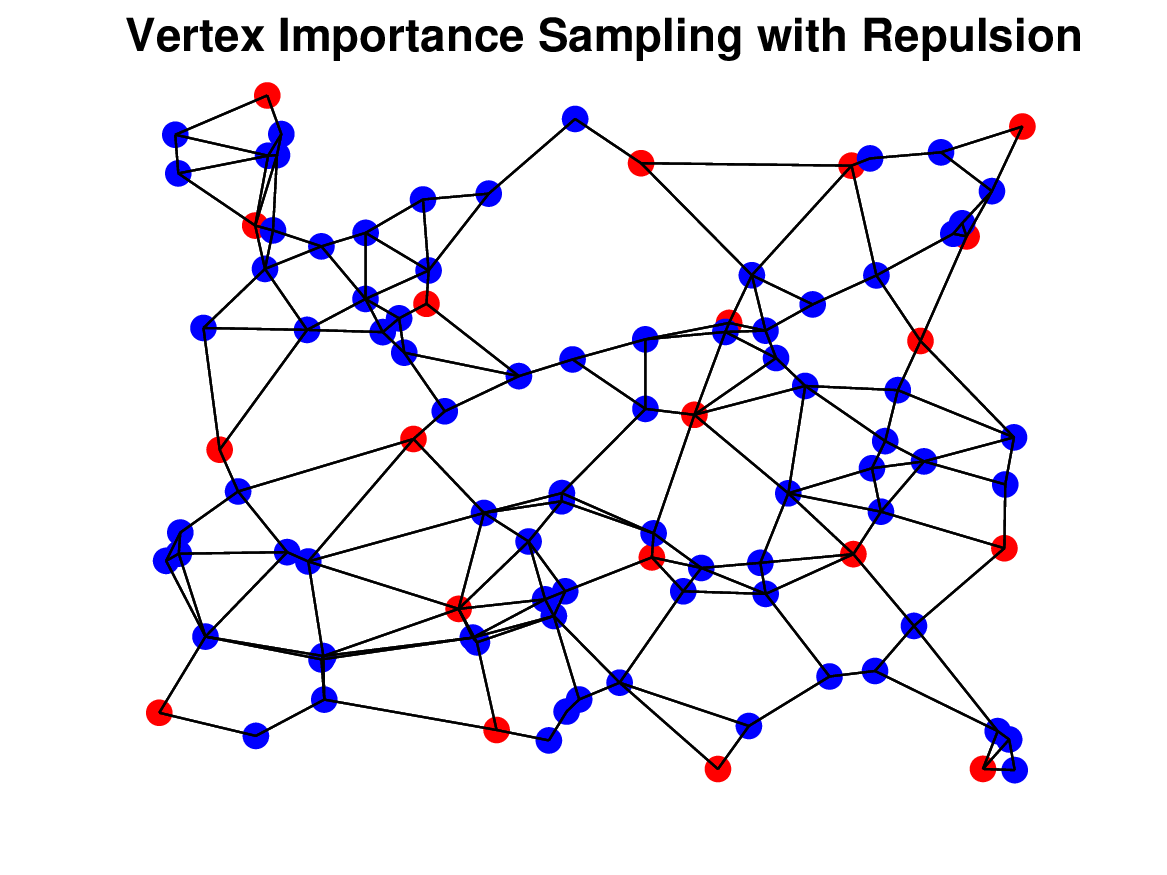}
         \caption{$\abs{\Sc} = 20$}
         \label{subfig:ISR_sampling}
     \end{subfigure}
\caption{Visual comparison of vertices selected. (a) learned graph with vertex importance shown in the color bar, (b) sampling set from VIS, (c) sampling set from VISR }
\label{fig:sampling_set_comparison}
\end{figure*}

We generate $M = 100$ random signals, each drawn from multivariate Gaussian distribution $\fv_{k} \sim \Nc(\zerov, \Sm)$, and add noise $\nv \sim \Nc(\zerov, \sigma_{n}^{2}\Id)$, where $\sigma_{n}$ is the noise level. We compare different sampling sets by reconstructing the sampled signal with noise using \eqref{eqn:reconstruction}. We evaluate the average MSE between reconstructed and original signals as follows,
\begin{equation}
    \label{eqn:avg_MSE}
    \overline{\mathrm{MSE}} = \frac{1}{M}\sum_{k = 1}^{M} [\frac{1}{N} \norm{\hat{\fv}_{k} - \fv_{k}}^{2}].
\end{equation}

\subsection{Comparison between VIS and VISR}
We evaluate the performance of VIS and VISR compared to the greedy solution to the D-optimal sampling objective presented in \autoref{subsec:greedy_sampl}. We present the sampling set from greedy solution to D-optimal objective in \autoref{subfig:d_opt_sampl}. We learn the DDGL using \eqref{eqn:LQ_learning} and show the result of the joint graph and vertex importance learning and sampling sets obtained from VIS and VISR in \autoref{fig:sampling_set_comparison}. 
As shown in \autoref{subfig:sampling_from_Q}, sampling from VIS results in spatially close samples since we sample every node independently. This is consistent with our analysis in \autoref{sec:sampling_from_q}. 
Conversely, the samples from VISR, as shown in \autoref{subfig:ISR_sampling} are spatially separated. 
We observe a significant intersection between the sampling sets obtained from VIS (\autoref{subfig:sampling_from_Q}) and greedy D-optimal sampling (\autoref{subfig:d_opt_sampl}), which supports our claim that indices corresponding to the highest values of vertex importance will approximately solve the D-optimal sampling objective. We present a quantitative evaluation of reconstruction accuracy when the samples are obtained using VIS, VISR, and greedy solution to D-optimal objective in \autoref{subfig:VIS_vs_VISR}. At higher sampling rates, VISR sampling achieves smaller $\overline{\mathrm{MSE}}$ than VIS and the greedy D-optimal solution.


\begin{figure*}[ht]
     \centering
        \begin{subfigure}[b]{0.33\textwidth}
            \centering
            \includegraphics[width=\linewidth]{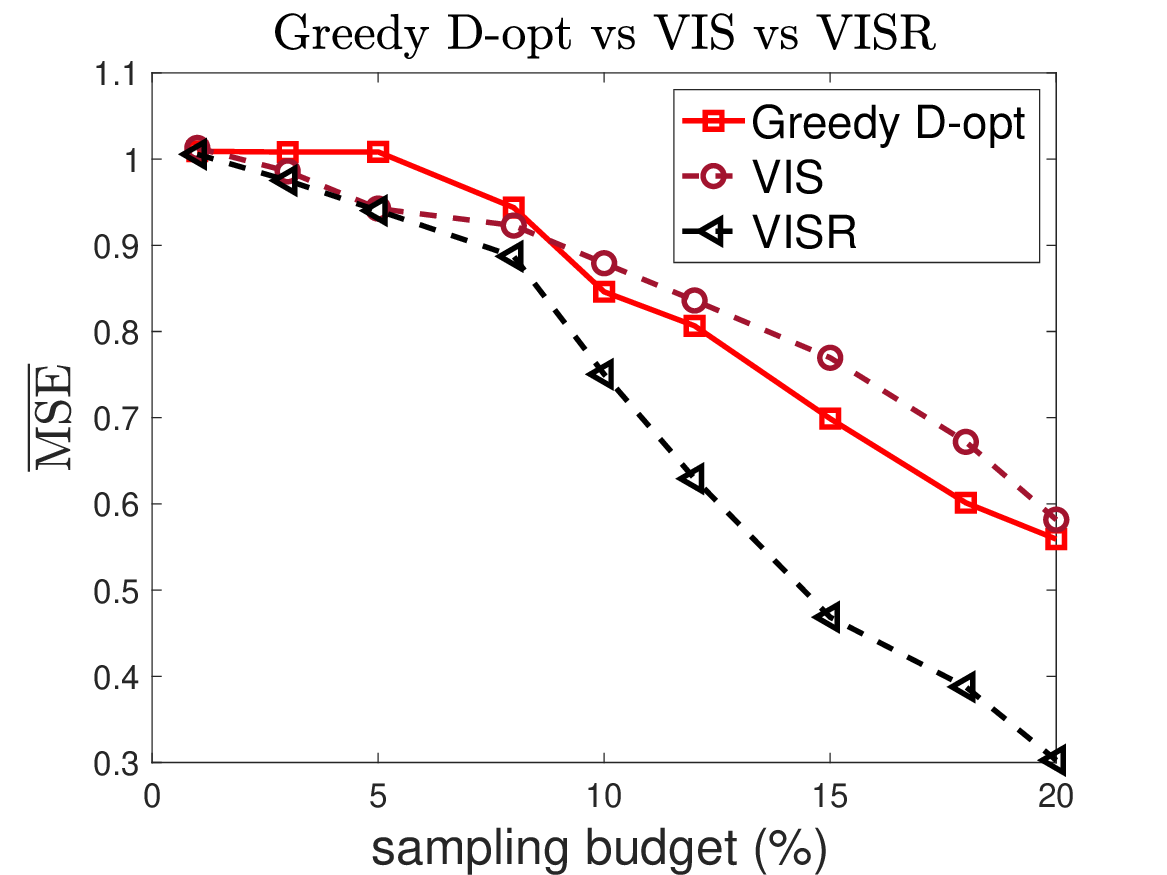}
            \caption{}
            \label{subfig:VIS_vs_VISR}
        \end{subfigure}
     \begin{subfigure}[b]{0.33\textwidth}
         \centering
         \includegraphics[width=\linewidth]{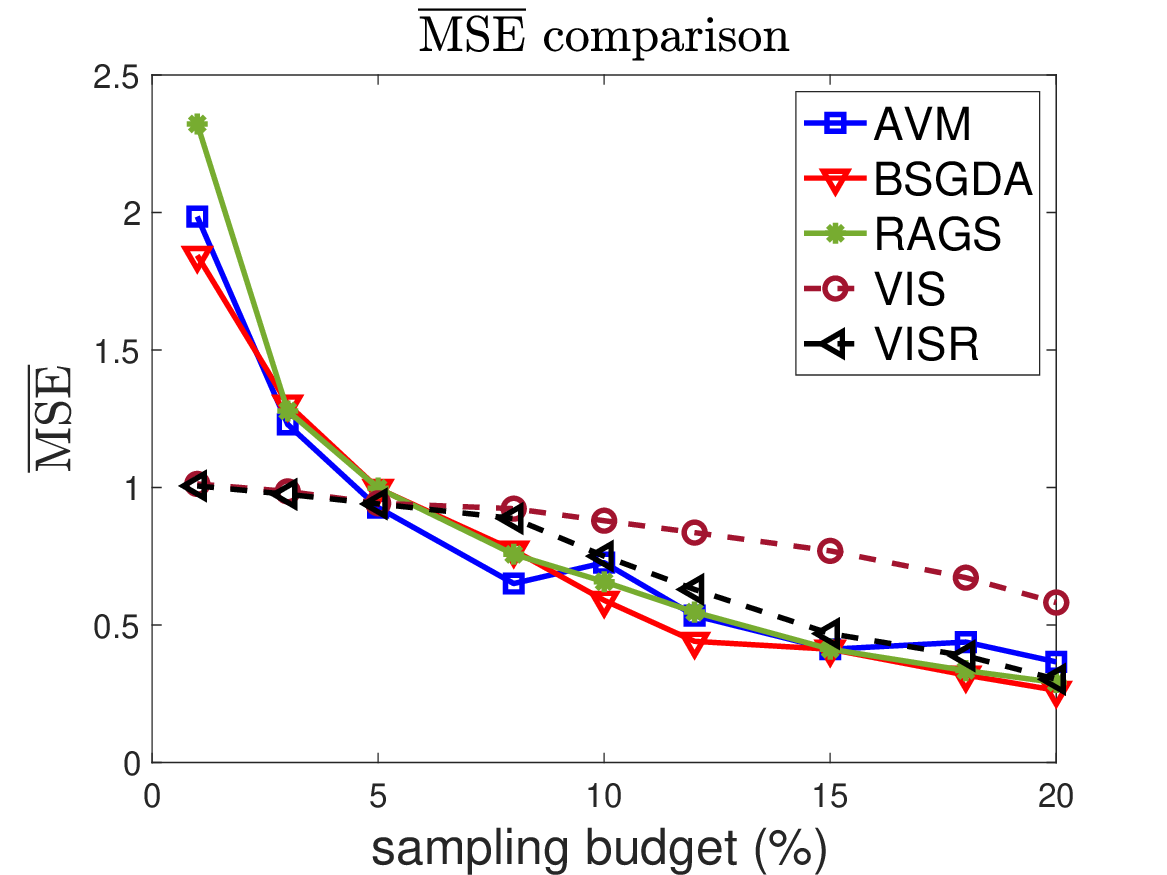}
         \caption{}
         \label{subfig:mse_sota}
     \end{subfigure}
     \begin{subfigure}[b]{0.33\textwidth}
            \centering
            \includegraphics[width=\linewidth]{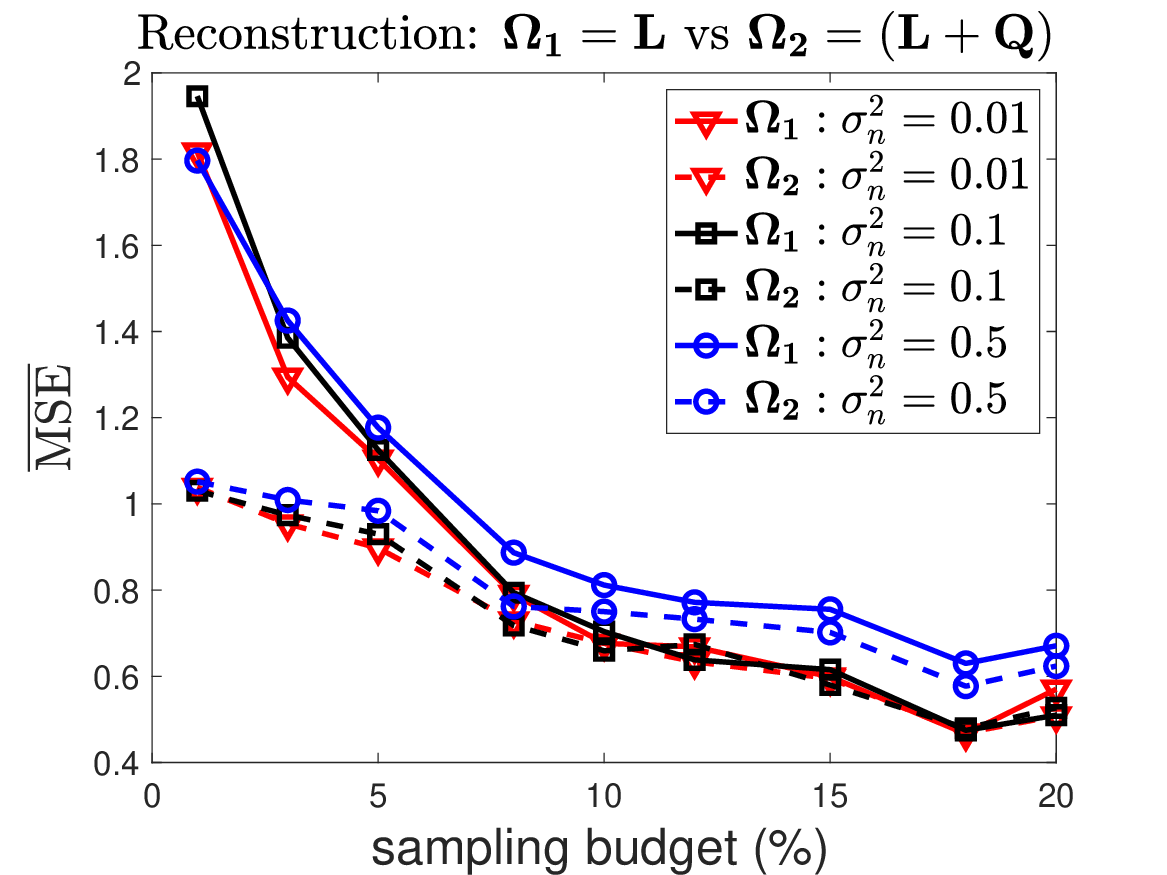}
            \caption{}
            \label{subfig:L_vs_LQ}
        \end{subfigure}
\caption{$\overline{\mathrm{MSE}}$ comparison in different settings. (a) comparison between VIS and VISR with greedy solution to D-optimal sampling objective, (b) comparison with other state-of-the-art sampling algorithms, (c) comparison of using $\Lm$ and $(\Lm + \Qm)$ for reconstruction from a randomly selected sampling set.}
\label{fig:mse_results}
\end{figure*}

\subsection{Comparison with other sampling algorithms}
We compare the proposed VIS and VISR with three state-of-the-art sampling set selection methods, AVM \cite{jayawant2021_avm}, BSGDA \cite{bai2020_bsgda}, and RAGS \cite{sridhara2024_pcsampling}. 
For the baselines, we follow the conventional two-step process, i.e., learn the CGL (without vertex importance) from data by solving \eqref{eqn:L_learning} and then employ AVM, BSGDA, and RAGS to obtain sampling sets. 
We reconstruct the sampled signal using \eqref{eqn:reconstruction} with $\Omegam_{1} = \Lm$. For VIS and VISR, we learn $(\Lm + \Qm)$ from \eqref{eqn:LQ_learning}, obtain the sampling sets as proposed in \autoref{sec:sampling_from_q}, and reconstruct the sampled signal using \eqref{eqn:reconstruction} with $\Omegam_{2} = (\Lm + \Qm)$. 
All sampling set selection algorithms are compared using the same data with increasing sampling
budgets. 
The results are shown in \autoref{subfig:mse_sota}. We observe that VIS and VISR achieve reconstruction accuracy that outperforms existing sampling algorithms by a factor of $2 \times$ at lower sampling budgets. 
As we increase the sampling budget, VIS results in subpar reconstruction accuracy. Conversely, VISR results in competitive reconstruction accuracy compared to existing sampling algorithms at higher sampling rates.

We further study the effect of using $\Omegam_{1} = \Lm$ and $\Omegam_{2} = (\Lm + \Qm)$ for reconstruction. To this end, we randomly sample the nodes, i.e, each node $i \in \Vc$ is assigned independently to the sampling set with probability $p = 1/N$, and reconstruct the signal using $\Omegam_{1}$ and $\Omegam_{2}$. 
We generated signal $\fv_{k}$, add noise $\nv$ with different noise levels $\sigma_{n}$ and compare $\overline{\mathrm{MSE}}$. From the results in \autoref{subfig:L_vs_LQ}, it is clear that  $(\Lm + \Qm)$ is a better model than $\Lm$ for graph signal reconstruction for all noise levels. Therefore, we can attribute the superior performance of VIS and VISR at a low sampling budget to using a better signal reconstruction model.

\section{Conclusion and Future work}
\label{sec:conclusion}
In this work, we showed that a sampling set can be directly obtained from the learned vertex importance. Compared to the conventional two-step approach of first learning the graph (CGL) and applying expensive sampling algorithms, our proposed approach of learning DDGL and using the vertex importance as auxiliary information can reduce computations and simplify sampling set selection. 
We proposed two sampling algorithms--- VIS and VISR and results show promising reconstruction accuracy compared to the existing sampling approaches. Building on this work, in future work we aim to combine the D-optimal sampling objective and DDGL learning into a single framework to learn efficiently the sampling set along with the graph. We will also perform a detailed study on the effect of vertex importance in graph signal reconstruction.


\bibliographystyle{IEEEbib}
\bibliography{strings,refs}

\end{document}